\documentclass[nohyperref]{article}

\usepackage{microtype}
\usepackage{graphicx}
\usepackage{subfigure}
\usepackage{booktabs} 

\usepackage{hyperref}

\usepackage[accepted]{icml2022}

\usepackage{amsmath}
\usepackage{amssymb}
\usepackage{mathtools}
\usepackage{amsthm}

\usepackage[capitalize,noabbrev]{cleveref}

\theoremstyle{plain}
\newtheorem{theorem}{Theorem}[section]

\theoremstyle{definition}

\theoremstyle{remark}

\usepackage[textsize=tiny]{todonotes}

\usepackage{float}
\usepackage[shortlabels]{enumitem}

\usepackage{tikz}
\usetikzlibrary{positioning}
\usetikzlibrary{arrows}
\usetikzlibrary{arrows.meta}
\usepackage{pgfplots}
\pgfplotsset{compat=1.16}
\usepackage{xcolor}
\definecolor{myGray}{rgb}{0.65,0.65,0.65}

\DeclarePairedDelimiter\abs{\lvert}{\rvert}
\DeclarePairedDelimiter\norm{\lVert}{\rVert}

\icmltitlerunning{GNNs with Precomputed Node Features}

\begin{document}

\twocolumn[
\icmltitle{Graph Neural Networks with Precomputed Node Features}



\icmlsetsymbol{equal}{*}

\begin{icmlauthorlist}
\icmlauthor{Beni Egressy}{yyy}
\icmlauthor{Roger Wattenhofer}{yyy}
\end{icmlauthorlist}

\icmlaffiliation{yyy}{ETH Zurich, Zurich, Switzerland}

\icmlcorrespondingauthor{Beni Egressy}{begressy@ethz.ch}


\vskip 0.3in
]

\begin{abstract}
Most Graph Neural Networks (GNNs) cannot distinguish some graphs or indeed some pairs of nodes within a graph. This makes it impossible to solve certain classification tasks. However, adding additional node features to these models can resolve this problem.
We introduce several such augmentations, 
including (i) positional node embeddings, (ii) canonical node IDs, and (iii) random features. 
These extensions are motivated by theoretical results and corroborated by extensive testing on synthetic subgraph detection tasks.
We find that positional embeddings significantly outperform other extensions in these tasks.
Moreover, positional embeddings have better sample efficiency, perform well on different graph distributions and even outperform learning with ground truth node positions. 
Finally, we show that the different augmentations perform competitively on established GNN benchmarks, and advise on when to use them.
\end{abstract}

\section{Introduction}


Given the impressive success of neural networks in the text and image domains, recently research has also turned its attention to graph-structured data. In just a few years there has been explosive interest in the area, with Graph Neural Networks (GNNs) achieving state-of-the-art results in a wide range of applications, including molecule recognition, physics simulations, recommendation systems, fake news detection and social networks \cite{fout2017protein, sanchez2020physics, ying2018recommender_systems}.

Most GNN architectures are based on the message passing framework, which can be summarized in three main steps: (1) node representations are initialized with their initial features (if available) or node degrees (if not available); (2) nodes update their representations by aggregating the representations of neighboring nodes; (3) the final representations of nodes are combined in a readout layer to solve the task at hand.\looseness=-1

Although this framework has been a recipe for success, it has also been observed that such GNNs are limited in their power to distinguish even very simple graphs, e.g., see Figure~\ref{fig:WL_example_labels} (top, ignoring the labels). This has been formalized by proving that message-passing GNNs are upper-bounded by the Weisfeiler-Lehman (WL) isomorphism test \cite{weisfeiler1968WL_test, xu2018powerful_gin, morris2019weisfeiler_WL}. We shall refer to such GNNs as Weisfeiler-Lehman GNNs, or WLGNNs.\footnote{We use the term WLGNNs as used in \cite{li2020distance_encoding}} 
More precisely, WLGNNs produce identical representations for (sub-)graphs that the WL test fails to distinguish. 
In particular these include regular graphs, where nodes cannot be distinguished based on their degree and the degree of the nodes around them.

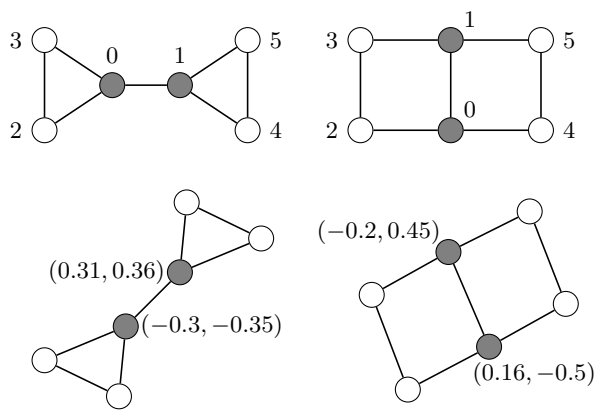
\begin{figure}
\centering
\begin{tikzpicture}[scale=0.3, every node/.style={scale=1.0}]
  \begin{scope}
    \node[circle, draw, minimum size=0.3cm] (A) at  (0,0) {};
    \node[circle, draw, minimum size=0.3cm] (B) at  (0,4) {};
    \node[circle, draw, minimum size=0.3cm, fill=gray] (C) at  (3,2) {};
    \node[circle, draw, minimum size=0.3cm, fill=gray] (D) at  (6,2) {};
    \node[circle, draw, minimum size=0.3cm] (E) at  (9,0) {};
    \node[circle, draw, minimum size=0.3cm] (F) at  (9,4) {};
    \draw [semithick,-] (A) -- (B);
    \draw [semithick,-] (A) -- (C);
    \draw [semithick,-] (B) -- (C);
    \draw [semithick,-] (C) -- (D);
    \draw [semithick,-] (D) -- (E);
    \draw [semithick,-] (D) -- (F);
    \draw [semithick,-] (E) -- (F);
    \node[rectangle, left = -0.0cm of A] {\footnotesize $2$};
    \node[rectangle, left = -0.0cm of B] {\footnotesize $3$};
    \node[rectangle, above = -0.0cm of C] {\footnotesize $0$};
    \node[rectangle, above = -0.0cm of D] {\footnotesize $1$};
    \node[rectangle, right = -0.0cm of E] {\footnotesize $4$};
    \node[rectangle, right = -0.0cm of F] {\footnotesize $5$};
    
    \node[circle, draw, minimum size=0.3cm] (A2) at  (14,0) {};
    \node[circle, draw, minimum size=0.3cm] (B2) at  (14,4) {};
    \node[circle, draw, minimum size=0.3cm, fill=gray] (C2) at  (18,0) {};
    \node[circle, draw, minimum size=0.3cm, fill=gray] (D2) at  (18,4) {};
    \node[circle, draw, minimum size=0.3cm] (E2) at  (22,0) {};
    \node[circle, draw, minimum size=0.3cm] (F2) at  (22,4) {};
    \draw [semithick,-] (A2) -- (B2);
    \draw [semithick,-] (A2) -- (C2);
    \draw [semithick,-] (B2) -- (D2);
    \draw [semithick,-] (C2) -- (D2);
    \draw [semithick,-] (C2) -- (E2);
    \draw [semithick,-] (D2) -- (F2);
    \draw [semithick,-] (E2) -- (F2);
    \node[rectangle, left = -0.0cm of A2] {\footnotesize $2$};
    \node[rectangle, left = -0.0cm of B2] {\footnotesize $3$};
    \node[rectangle, above right = -0.1cm of C2] {\footnotesize $0$};
    \node[rectangle, above right = -0.1cm of D2] {\footnotesize $1$};
    \node[rectangle, right = -0.0cm of E2] {\footnotesize $4$};
    \node[rectangle, right = -0.0cm of F2] {\footnotesize $5$};
  \end{scope}
  
  \begin{scope}[yshift=-7.5cm, xshift=4.8cm]
    \node[circle, draw, minimum size=0.3cm] (A) at  (-4.8,-2.7) {};
    \node[circle, draw, minimum size=0.3cm] (B) at  (-1.5,-4.3) {};
    \node[circle, draw, minimum size=0.3cm, fill=gray] (C) at  (-1.2,-1.2) {};
    \node[circle, draw, minimum size=0.3cm, fill=gray] (D) at  (1.2,1.2) {};
    \node[circle, draw, minimum size=0.3cm] (E) at  (4.8,2.7) {};
    \node[circle, draw, minimum size=0.3cm] (F) at  (1.5,4.3) {};
    \draw [semithick,-] (A) -- (B);
    \draw [semithick,-] (A) -- (C);
    \draw [semithick,-] (B) -- (C);
    \draw [semithick,-] (C) -- (D);
    \draw [semithick,-] (D) -- (E);
    \draw [semithick,-] (D) -- (F);
    \draw [semithick,-] (E) -- (F);
    \node[rectangle, right = -0.1cm of C] {\footnotesize $(-0.3, -0.35)$};
    \node[rectangle, left = -0.1cm of D] {\footnotesize $(0.31, 0.36)$};
    
    \node[circle, draw, minimum size=0.3cm] (A2) at  (9.7,0.2) {};
    \node[circle, draw, minimum size=0.3cm] (B2) at  (11.3,-4.0) {};
    \node[circle, draw, minimum size=0.3cm, fill=gray] (C2) at  (13.1,2.1) {};
    \node[circle, draw, minimum size=0.3cm, fill=gray] (D2) at  (14.9,-2.1) {};
    \node[circle, draw, minimum size=0.3cm] (E2) at  (16.7,3.9) {};
    \node[circle, draw, minimum size=0.3cm] (F2) at  (18.3,-0.2) {};
    \draw [semithick,-] (A2) -- (B2);
    \draw [semithick,-] (A2) -- (C2);
    \draw [semithick,-] (B2) -- (D2);
    \draw [semithick,-] (C2) -- (D2);
    \draw [semithick,-] (C2) -- (E2);
    \draw [semithick,-] (D2) -- (F2);
    \draw [semithick,-] (E2) -- (F2);
    \node[rectangle, above left = -0.2cm of C2] {\footnotesize $(-0.2, 0.45)$};
    \node[rectangle] at (16.9,-3.3) {\footnotesize $(0.16, -0.5)$};
  \end{scope}
\end{tikzpicture}
\caption{Example of limitation of WLGNNs, which cannot distinguish the graphs on left and right without node features. Both our augmentations, canonical ID (above) and positional embedding (below) resolve this issue. For example the two nodes with canonical ID $0$ (top) now have different $1$-hop neighborhoods (${1,2,3}$ vs ${1,2,4}$) and can be distinguished by $2$ rounds of message passing. For the positional embeddings, we only show the coordinates of select nodes, but the positions correspond to the embeddings.}
\label{fig:WL_example_labels}
\end{figure}

As a consequence there have been various suggestions to overcome the expressive limit of WLGNNs. These include: breaking symmetries by making perturbations to the input graph; using higher-order representations of the input graph; and initializing the nodes with 
additional features.

The latter approach will be the focus of this paper. 
We explore the power of GNNs augmented with precomputed node features. Since initializing node representations is the first main step of most standard GNNs, this approach can be combined with many different architectures.


Our main contributions are:
\begin{itemize}
\itemsep0em 
    \item We propose several new augmentations for standard GNN architectures to improve their expressive power beyond the WL test: canonical node IDs, $l_{\infty}$ embeddings, positional embeddings, random bits, and combinations of the above. See Figure \ref{fig:WL_example_labels}. 
    \item For several augmentations, we formally prove that they make appropriate GNN architectures universal.
    \item We conduct extensive subgraph detection tests on carefully constructed benchmark datasets to show the empirical validity of these augmentations. 
    \item We show that the augmentations perform competitively on established benchmarks.
\end{itemize}








\section{Preliminaries}

\subsection{Notation}

We consider unweighted, undirected, connected graphs $G=(V,E)$ consisting of $n=|V|$ nodes. We denote an edge between nodes $u$ and $v$ by $e_{uv}$. We denote the neighborhood of node $v$ by $N(v) = \{u \in G \mid e_{uv} \in E\}$, the degree of node $v$ by $\deg(v) = |N(v)|$, and the maximum degree of all nodes by $\Delta = \max_{v \in V} \deg(v)$. The graph diameter $D$ is the length of the longest shortest path between any two nodes.\looseness=-1

\subsection{GNNs}

Most GNN architectures are based on the message passing framework, which consists of: (1) node initialization, (2) node updates via message passing, and (3) a readout step. 
The message passing itself is conducted in rounds. In each round there are three steps. First, every node creates a \textsc{message} based on its current embedding. It then sends this message to each of its neighbors. Second, when the nodes have received messages from each of their neighbors, they \textsc{aggregate} these messages. Finally they \textsc{update} their state by combining their current state with the aggregate embedding they have just calculated. One round of message passing corresponds to one layer of a GNN. Usually $k$ layers are stacked on top of each other so that after $k$ rounds of message passing each node will have received information, either directly or indirectly, from all of the nodes in its $k$-hop neighborhood.   
Usually, \textsc{message}, \textsc{aggregate}, \textsc{update} and \textsc{readout} are functions with learnable parameters, and they are shared among all the nodes. In the simple case, where the messages are the current state of the node, the GNN can be summarized as follows:
\begin{align*}
    h_{v}^{(0)} &= x_{v} \\
    a_{v}^{(t)} &= \textsc{aggregate} ( \{ h_{u}^{(t-1)} \mid u \in N(v) \} ) \\
    h_{v}^{(t)} &= \textsc{update} ( h_{v}^{(t-1)}, a_{v}^{(t)} ) \\
    y_{v} &= \textsc{readout} (h_{v}^{(t)})
\end{align*}
where $x_{v}$ are the initial node features and $h_{v}^{(t)}$ is the embedding vector, or representation, of node $v$ after $t$ rounds of message passing. The readout layer above is a function of a single node in the case of node classification, but could also be a function of all the nodes (graph classification), or a subset of the nodes (link prediction). 

Since reordering the nodes of a graph does not change the underlying topology, in general we want to learn a graph or node function that is invariant to permutations of the nodes. This is commonly ensured by treating incoming messages as a set (or multiset) and learning a (permutation-invariant) set aggregation function for \textsc{aggregate}.


\subsection{Weisfeiler-Lehman Test}

The message passing framework is closely related to the Weisfeiler-Lehman (WL) isomorphism test. The WL test is an iterative color refinement procedure \cite{weisfeiler1968WL_test, shervashidze2011weisfeilerWL}. Each node keeps a state (or color) that gets refined in each iteration by aggregating information from the neighbors’ states and combining it with its own. 

More precisely, in each iteration we assign to each node a tuple containing the node’s state and a multiset of the node’s neighbors' states. Then we hash these tuples to give each node its new state. In this way all nodes with different tuples receive different new states and all nodes with the same tuple receive the same new state. In particular two nodes with different states at time $T \geq 0$ will never have the same state after time $T$; hence we call this a \textit{refinement} algorithm. All nodes are assigned the same initial state, so for example, after the first iteration the state of a node corresponds to its degree. To check whether two graphs are isomorphic we can run this procedure on both graphs until convergence and then compare the multisets of the final node states. If they are different, then the graphs are not isomorphic. If they are the same, then we cannot be sure. 

If the WL test cannot distinguish two graphs, then a standard message-passing GNN cannot distinguish them either: intuitively, the nodes in the two graphs receive the same messages and create the same embeddings in each round, and thus they always arrive at the same final result. As such, a message-passing algorithm can only be as powerful\footnote{as good at distinguishing non-isomorphic graphs} as the WL algorithm. 
In fact for some GNNs, it can be shown that they are exactly as powerful as the WL isomorphism test. This is the case with GIN \cite{xu2018powerful_gin}, which we shall rely on in this paper.

The WL test can be extended to start with an initial labelling of the nodes. If the nodes have initial features, then these can be hashed to give the initial states before continuing the algorithm. Similarly, initial features can be used to increase the power of a GNN. For example, in chemistry problems, the nodes might represent atoms of a molecule, and so the type of atom would be an initial feature for each node \cite{irwin2012zinc, yanardag2015real_datasets}.

\subsection{Distributed Computing Models}

The GNN message passing framework is also closely related to message passing models in distributed computing\footnote{The connection between GNNs and distributed computing models was first noted in \cite{loukas2019graph_node_ids_proof}.}. 

As GNNs, distributed computing deals with a network of nodes, connected by message passing edges, and the task is to calculate some function of the graph. As GNNs, all nodes also execute the same algorithm; however, the algorithm is designed rather than learned.


One major difference to GNNs is that in distributed computing, the nodes usually have unique IDs.\footnote{There are exceptions, in particular anonymous networks \cite{seidel}.} This ensures that a node can always distinguish its neighbors and indeed all nodes in the graph. If computation and communication is unbounded, any problem can be solved in $D$ rounds of message passing, where $D$ is the diameter of the graph. Each node can simply encode all the information it has about the graph so far into a message and send this information to all of its neighbors in every round. With the help of the unique IDs, after $D$ rounds, every node will know the exact topology of the graph. 

\section{Motivation}

Upon considering the aforementioned distributed computing model, one natural way to increase the power of GNNs is to treat the node indices as node IDs and use the one-hot encoding of node indices as an input feature. This way the GNN can distinguish all the nodes, and it can theoretically distinguish all non-isomorphic graphs, making it ``universal''. However, the embedding will depend heavily on the initial ordering of the vertices, and there are $n!$ possible permutations. As such the GNN loses its inductive advantage of mapping nodes with identical neighborhoods to the same embedding, resulting in a great loss in generalization ability. 
Similarly, one could use random features as input node features \cite{abboud2020surprising, sato2021random}. This can also be proven to give ``universal'' GNNs, but it also leads to the same generalization problems. 
Basically,  we want to give each node a label such that:\looseness=-1
\begin{enumerate}[(a)]
    \item the probability to have the same label(s) is high if the graph is the same, and \label{R1}
    \item the probability to have the same label(s) is low if the graph is not the same. \label{R2}
\end{enumerate}


Using no labels (or identical labels), as in the WL algorithm, perfectly achieves \ref{R1}, but does not do well regarding \ref{R2}, since many graphs remain indistinguishable. On the other hand, using node indices or random features as labels may achieve ``universality'' \ref{R2}, but does less well regarding \ref{R1}. The problem is that the additional information is not permutation invariant. As a consequence, exponentially more training examples might be needed. 
The additional information can also confuse the learning procedure, and the GNN may weigh the additional information too heavily in the decision process.

We hypothesise that the right balance of \ref{R1} and \ref{R2} is at the core of a deeper understanding of GNNs. By keeping the additional information low, we get a net benefit from adding it. 
We explore this hypothesis by testing different augmentations aimed at balancing \ref{R1} and \ref{R2}. 

First, we propose adding a vector of random bits. By controlling the length of the vector we have a fine grained control over \ref{R1} and \ref{R2}. 
Second, we propose precomputed \emph{canonical node IDs}; these do not burden the learner with an exponential input diversity, while at the same time offering benefits regarding \ref{R2}.

The most impressive results are obtained by 
our novel precomputed \emph{positional node embeddings}. Like random features, these geometric labels also fulfill \ref{R2} to a sufficient degree. 
While positional embeddings add even more information than node IDs, in practice GNNs can interpret positional coordinates well. GNNs seem to develop an understanding of these coordinates (e.g., the distance between two positions), without being distracted by the erratic raw values of the positions.

\subsection{Canonical node IDs}

Canonical node IDs are node IDs that are standardized such that all the nodes in a graph always receive the same IDs however the nodes are permuted at the input. 
It is not clear how we can find canonical node IDs fast in practise. Indeed, an algorithm that produces canonical node IDs also solves the graph isomorphism problem, and until recently it was not known whether the graph isomorphism problem can be solved in sub-exponential time. However, Babai et al. showed that solving graph isomorphism and finding a canonical form can be done in quasipolynomial time \cite{helfgott2017graphisomorphism, babai2019canonical_graph}. However, their approach is mostly of academic interest, since in practice the overhead is too high. Instead, \textsc{nauty} \cite{mckay2014nauty} and \textsc{bliss} \cite{junttila2007bliss} are used in practice. These algorithms do not guarantee a sub-exponential runtime, but are fast in practice. 
We will use \textsc{nauty} when calculating canonical orderings in this paper. See Figure \ref{fig:WL_example_labels} for example graphs with canonical IDs.

First we show that GNNs with canonical node IDs can theoretically distinguish any two non-isomorphic graphs. The proof follows from Corollary 3.1 in \cite{loukas2019graph_node_ids_proof}. As such we leave the proof for the appendix.

\begin{theorem}
\label{thm:GNN_canon}
GIN of sufficient depth augmented with canonical node IDs can distinguish any two non-isomorphic graphs.\looseness=-1
\end{theorem}

In our empirical tests, we use the one-hot encoding of the canonical node IDs as initial node features. Note that using canonical node IDs is very similar to using the node indices as features. But crucially there are no longer $n!$ possible inputs for the same graph. Instead a graph will always receive the same labelling. We expect this to lead to better learning and sample efficiency. 

\subsection{Positional node Embeddings}

There have been recent attempts to improve GNNs by learning positional embeddings, see for example \cite{ma2021gat_with_pos, klemmer2021positional}. However, these are based on solving some auxiliary task while training the main network, with the former using a combined neural network, and the latter focusing on geographic data. 

Using precomputed positional embeddings has several advantages over learned embeddings.
Firstly, the node embedding algorithms we use take the whole graph topology into account. 
One could also use fully connected message passing to achieve this, but message passing GNNs have a known problem with oversmoothing and fully connected layers only exaggerate this problem.  
Secondly, the positional embedding algorithms we use can guarantee unique node positions, leading to universal GNNs.
Finally, precomputed node features are also very versatile; the node features can be used to augment any standard GNN.

\subsubsection*{$l_{\infty}$-embeddings}

We propose embedding the input graph into $l_{\infty}$ and using the positions as node features. We begin by showing that any graph can be embedded into $l_{\infty}$ isometrically (i.e., such that all shortest path distances are preserved). 
This guarantees that no information about the graph is lost.
One could throw away the graph topology and the set of embeddings would still hold all the information.

\begin{theorem}
\label{thm:l_inf}
Every graph $G=(V,E)$ embeds isometrically into $l_{\infty}$.
\end{theorem}

See the appendix for a proof. Unfortunately the embedding in the proof uses $n$ dimensions, and adding node features of length $n$ is impractical for large graphs. It can be shown that $\Theta(n)$ is in fact a lower bound for the dimension of isometric embeddings into $l_{\infty}$ \cite{coppersmith2001infinity}. For other $l_{p}$ spaces, isometric embeddings are in general not even possible.

We use the isometric $l_{\infty}$-embeddings from the proof in our synthetic tasks. These graphs are small enough for this approach to be practical. Moreover, we can prove that GNNs with $l_{\infty}$-embeddings are universal. The proof follows the same approach as Theorem \ref{thm:GNN_canon}, see appendix.

\begin{theorem}
\label{thm:GNN_l_inf_universal}
GIN of sufficient depth augmented with isometric $l_{\infty}$-embeddings can distinguish any two non-isomorphic graphs.
\end{theorem}



Although these embeddings do not lose any information, they have a major drawback: The embeddings are large and their dimension depends on the number of nodes. If we want our model to accommodate
graphs of up to $n_{\text{max}}$ nodes, then we need to use node labels of length $n_{\text{max}}$. This is not only inconvenient, but also makes the learning problem much harder. As an alternative we introduce approximate Euclidean embeddings.

\subsubsection*{Positional Embeddings}

To calculate our positional embeddings we use a force-directed technique based on an approach proposed in \cite{kamada1989algorithm}. The main idea is to obtain a layout of the graph minimizing the following \textsc{stress} function:

$$ \textsc{stress}(P) = \sum_{i<j} w_{ij} \left( ||p(v_i) - p(v_j)||_{2} - d_{ij} \right)^2 , $$

where $p(v_i)$ is the coordinate position of vertex $v_i$ in the Euclidean vector space; $d_{ij}$ is the shortest path distance between nodes $v_i$ and $v_j$; and $w_{ij}$ is a weighting factor for balancing the influence of certain pairs of nodes. We use $w_{ij} = d_{ij}^{-2}$. The localized Newton-Raphson method is used to optimize the stress function. This is an iterative solver, where in each step a single node is chosen and the standard N-R method is used to find a local minimum of the stress function. See Figure \ref{fig:WL_example_labels} for example graphs with labels.

One can think of this approach as placing an ideal spring between every pair of nodes with length equal to the shortest path distance between the endpoints. The springs are then used to push/pull the nodes so that their euclidean distance
approximates their shortest path distance in the graph. 

These algorithms are randomized so they do not produce the same embeddings when run multiple times\footnote{We run the embedding algorithm a single time for each input graph, and thereafter keep these embeddings fixed during training.}. However, the relative positions between nodes are always optimized. 
Intuitively, this is something that should help with generalization to unseen graphs. Regardless of the exact placement of nodes, relative distances should be preserved and unlike random features or canonical node labels, these encode structural information about the nodes.


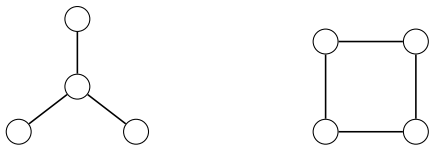
\begin{figure}[H]
\centering
\begin{tikzpicture}[scale=0.3, every node/.style={scale=1.0}]
    \node[circle, draw, minimum size=0.3cm] (A) at  (0.4,0) {};
    \node[circle, draw, minimum size=0.3cm] (B) at  (3,5) {};
    \node[circle, draw, minimum size=0.3cm] (O) at  (3,2) {};
    \node[circle, draw, minimum size=0.3cm] (C) at  (5.6,0) {};
    \draw [semithick,-] (A) -- (O);
    \draw [semithick,-] (B) -- (O);
    \draw [semithick,-] (C) -- (O);
    
    \node[circle, draw, minimum size=0.3cm] (A2) at  (14,0) {};
    \node[circle, draw, minimum size=0.3cm] (B2) at  (14,4) {};
    \node[circle, draw, minimum size=0.3cm] (C2) at  (18,4) {};
    \node[circle, draw, minimum size=0.3cm] (D2) at  (18,0) {};
    \draw [semithick,-] (A2) -- (B2);
    \draw [semithick,-] (B2) -- (C2);
    \draw [semithick,-] (C2) -- (D2);
    \draw [semithick,-] (D2) -- (A2);
\end{tikzpicture}
\caption{Example graphs that cannot be embedded into Euclidean space (of any dimension).}
\label{fig:non_eucl_embed_example}
\end{figure}

    

In general, unlike with $l_{\infty}$, graphs cannot always be isometrically embedded into Euclidean space, see Figure \ref{fig:non_eucl_embed_example}. Moreover there can be equilibrium points, where different nodes receive the same position.
This means that we cannot directly apply the same universality result that we used in Theorem \ref{thm:GNN_canon}.
However this is highly unlikely to occur in multiple dimensions. 
Moreover, there are several ways to guarantee unique vertex positions. One can adapt the algorithm to avoid collisions explicitly \cite{bostock2011d3} or one can apply post-processing methods to remove node overlaps \cite{huang2007non-overlap, marriott2003removing-overlap}. 
Taking modifications into account, one can again theoretically guarantee that the associated augmented GNNs are universal function approximators. 
We did not apply such modifications in our experiments as we had no overlaps.

\subsection{Random Baselines}

Using random node features to initialize GNNs was first explored in \cite{sato2021random}, and was shown to give universal function approximators in \cite{abboud2020surprising}. In the latter, uniform and normal distributions, and in the former discrete uniform distributions are used to initialize the nodes. We initialize our random features in three different ways:\looseness=-1

\textbf{Random Gaussian:} 
In this model, we add a vector of $d$ random Gaussian samples to each node. The samples are generated i.i.d., from a $\mathcal{N}(0, 1)$ distribution. This is similar to \cite{abboud2020surprising}. We denote this as \emph{rNormal($d$)}.

\textbf{Random Uniform Integer:}
In this model, we assign a random integer from $0$ to $99$ to each node; denoted \emph{rUniform}.

\textbf{Randombits:}
In this model, we add a vector of $d$ random bits to each node. The random bits are generated i.i.d., from a Bernouilli(0.5) distribution. We denote this as \emph{rBits($d$)}.

Note that random features can be viewed as a form of training data augmentation. We generate the features freshly each time an input graph is presented to the model. As such the model sees many versions of the same training graph. This could be an advantage over the precalculated canonical and positional node features. 

\subsection{Our Model}

First we precalculate the additional node features for a given input graph, be they canonical node IDs, positional node embeddings, or random features. We then concatenate the precalculated node features with the provided node features (if available) and use these as input for our GNN. We mainly use the GIN architecture for consistency \cite{xu2018powerful_gin}. We choose GIN as it can be shown to be theoretically as powerful as the WL test.

\section{Experiments}

\subsection{Model Parameters}

Unless otherwise stated, we use GIN as our standard WLGNN. We use $5$ layers by default with hidden dimension $64$. We use weighted cross-entropy loss, Adam optimizer with initial learning rate 0.01 and decay of rate $0.5$ after every $50$ epochs. We use dropout on the final layer and we train for $300$ epochs by default with a batch size of $32$.

\subsection{Datasets beyond WL}

\begin{table*}
\centering
\resizebox{0.7\linewidth}{!}{%
    \begin{tabular}{lrrrrrrrr}
    \toprule
    model & C3 & C4 & C5 & C6 & K4 & K5 & K6 & LCC \\
    \midrule
    GIN & 0.500 & 0.500 & 0.500 & 0.500 & 0.500 & 0.500 & 0.500 & 0.500 \\
    +canon(20) & 0.943 & 0.605 & 0.544 & 0.560 & 0.673 & 0.594 & 0.577 & 0.843 \\
    +$l_{\infty}$(20) & 0.686 & 0.597 & 0.526 & 0.548 & 0.598 & 0.562 & 0.536 & 0.701 \\
    +pos(2) & \textbf{0.968} & \textbf{0.910} & 0.763 & \textbf{0.723} & 0.780 & 0.751 & 0.752 & 0.880 \\
    +pos(2)+rBits(1) & \textbf{0.968} & \textbf{0.910} & \textbf{0.773} & 0.722 & \textbf{0.783} & \textbf{0.755} & \textbf{0.758} & \textbf{0.884} \\
    +rUniform(1) & 0.913 & 0.675 & 0.593 & 0.598 & 0.551 & 0.519 & 0.514 & 0.686 \\
    +rNormal(1) & 0.898 & 0.652 & 0.587 & 0.609 & 0.553 & 0.512 & 0.508 & 0.685 \\
    +rBits(2) & 0.949 & 0.701 & 0.612 & 0.622 & 0.546 & 0.507 & 0.515 & 0.695 \\
    \midrule
    ppgn$^*$ \cite{maron2019ppgn} & 0.707 & 0.742 & 0.667 & 0.623 & 0.629 & 0.651 & 0.706 & 0.645 \\
    dropgnn \cite{papp2021dropgnn} & 0.665 & 0.671 & 0.671 & 0.624 & 0.655 & 0.675 & 0.658 & 0.660 \\
    \bottomrule
    \end{tabular}
}
\caption{AUROC scores on the synthetic tasks using GIN with $5$ layers as the base GNN model. Model name indicates the node feature used to augment the base model, with numbers in brackets indicating the size of the additional node features. All models were run with the same hyperparameter settings (epochs, learning rate, dropout, etc.). $^*$Ppgn had to be adapted for node classification.}
\label{tab:syn_gin_roc}
\end{table*}    

To systematically test the practical expressiveness of GNNs with the suggested node features, we propose a series of binary node classification tasks on regular graphs. The model is tasked with detecting a particular subgraph in the input, and outputting $1$ for all nodes that are in such a subgraph, and $0$ for all nodes that are not in such a subgraph. Since the graphs are regular, the tasks cannot be solved by standard WLGNNs; indeed all nodes look identical to the WL test.

\textbf{Ci\_N:} We generate N regular train and N regular test graphs, with ground truth labels indicating whether a node is in a cycle of length i.
This synthetic task is based on C3\_1000 from \cite{sato2021random}. 
The degree we use for generating these regular graphs is given by the degree that leads to the closest-to-even split of ground truth labels.\footnote{An exception to this is C3\_1000, where we use degree $3$ (instead of $4$) to match the degree from \cite{sato2021random}.}

\textbf{Ki\_N:} We generate N train and N test graphs, with ground truth labels indicating whether a node is in a clique of size i.\looseness=-1

\textbf{LCC\_N:} We generate N train and N test graphs, with ground truth labels indicating how many triangles a node is in. See \cite{sato2021random}.
\footnote{LCC, or local clustering coefficient, of a node in a graph quantifies how close its neighbors are to being a clique. This is given by the number of triangles divided by the number of potential triangles based on the number of neighbors. Here we do not divided by the number of potential triangles, so the labels are just a triangle counts.}

Note that we do not discard any randomly generated graphs based on the number of $0$ and $1$ labels. This ensures that our datasets do not have any prior biases resulting from the generation process. They come from the uniform distribution over all connected regular graphs of given degree with given size. 
Also note that this means that the GIN test accuracy baseline (with no additional node features) is given by the $max$ of the proportion of $0$'s and $1$'s. However the weighted AUROC (Area Under the Receiver Operating Characteristic Curve) baseline for WLGNNs is $0.5$. As such we will generally quote weighted AUROC scores, following \cite{sato2021random}. Please check the appendix for the accuracy scores. 

We ran all our augmentations individually and we combined the top two augmentation (positional embeddings and random bits) by concatenating the features. All GIN based models were run with the same hyperparameter settings. We compare against ppgn \cite{maron2019ppgn} and the recent DropGNN \cite{papp2021dropgnn}, both of them models specifically designed to go beyond the WL test. We adapted ppgn for node classification, as it was implemented as a graph classification model. See Table \ref{tab:syn_gin_roc} for the results. Most notable is that the positional embeddings outperform all the alternatives on all the datasets. We ran the same experiments with GCN \cite{kipf2016gcn} as our base model and got similar results, see appendix.


\subsection{Effect of Embedding Dimension}

We experiment with using larger embedding spaces for the positional embeddings. Additional dimensions allow for better separation of the nodes; this can be particularly advantageous for dense graphs. However, we find that $2$-dimensional embeddings perform the best across the synthetic tasks. See Figure \ref{fig:pos_dimension} for results on C3. 

\begin{figure}[h]
    \centering
    \includegraphics[width=\linewidth]{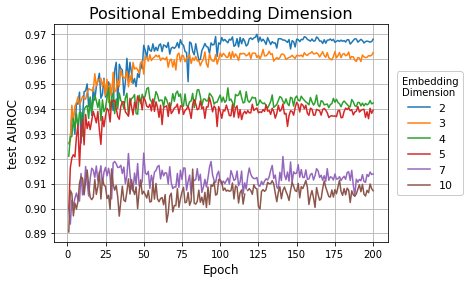}
    \vspace{-0.7cm}
    \caption{Weighted test AUROC scores for different positional embedding dimensions. Results for C3 task.}
    \label{fig:pos_dimension}
\end{figure}

\subsection{Investigating Learning Efficiency}

As noted in our description, we expect GNNs with canonical IDs to have a better learning efficiency, because the nodes of an input graph will always receive the same IDs. This is in contrast to the random augmentations. 
To confirm this, we plot the training loss per epoch for C3\_1000 in Figure \ref{fig:learning_efficiency_train_loss}. 
We see that the training loss for the model with canonical IDs drops the fastest and levels off the soonest. The model with positional node embeddings levels off next and the random augmentations level off last. See the appendix for a similar plot of the test AUROC.

\begin{figure}[h]
    \centering
    \includegraphics[width=0.9\linewidth]{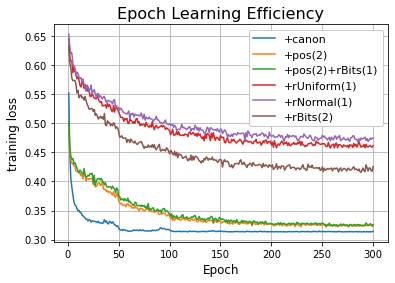}
    \vspace{-0.4cm}
    \caption{Training loss over the first 300 epochs on C3.}
    \label{fig:learning_efficiency_train_loss}
\end{figure}

\subsection{Investigating Sample Efficiency}

The models utilizing positional node embeddings perform significantly better on all the synthetic tasks with the limited training data (N=1000). This suggests that positional node embeddings, although not necessarily consistent between isomorphic graphs, store graph topological data is a somewhat consistent and generalizable way. Since neighboring nodes are placed closer together, it becomes possible to check whether two neighbors are themselves connected based on their positions. This, at least intuitively, might explain why GNNs with positional node embeddings are dramatically better at finding small cycles and cliques in regular graphs.

As we increase the training data available, the GNN with one-hot canonical node labels eventually ends up overtaking even the GNNs with positional node embeddings for the smaller cycle detection tasks, C3 and C4. We show results for C4 in Table \ref{tab:sample_efficiency_C4}. On all other synthetic tasks, the positional node embeddings still outperform the alternatives.


\begin{table}[ht]
\centering
\resizebox{\linewidth}{!}{%
\begin{tabular}{lrrrr}
\toprule
             model &  C4\_1000 &  C4\_2000 &  C4\_4000 &  C4\_8000 \\
\midrule
               GIN &    0.500 &    0.500 &    0.500 &    0.500 \\
        +canon(20) &    0.624 &    0.645 &    0.794 &    \textbf{0.970} \\
 +$l_{\infty}$(20) &    0.577 &    0.598 &    0.624 &    0.635 \\
           +pos(2) &    \textbf{0.904} &    \textbf{0.932} &   \textbf{ 0.945} &    0.949 \\
       rUniform(1) &    0.669 &    0.673 &    0.677 &    0.669 \\
        rNormal(1) &    0.669 &    0.665 &    0.666 &    0.667 \\
          rBits(2) &    0.689 &    0.708 &    0.696 &    0.709 \\
\bottomrule
\end{tabular}
}
\caption{Weighted AUROC scores with different training dataset sizes on the C4 cycle detection task.}
\label{tab:sample_efficiency_C4}
\end{table}

\subsection{Effect on Other Graph Distributions}

We check that the higher expressivity of our model with positional embeddings does not come at the cost of lower accuracy on other graph distributions, where standard GNNs already perform well. We repeat the synthetic tasks with Erdős-Renyi graphs in place of regular graphs. 
A graph is constructed by initializing $n$ nodes and connecting every pair of nodes independently with probability $p$. We choose $p$ separately for each task so as to balance the ground truth labels of the dataset. The results can be seen in Figure~\ref{fig:erdos}.

\begin{figure}[ht]
    \centering
    \includegraphics[width=0.9\linewidth]{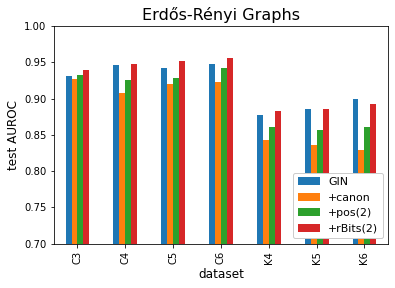}
    \vspace{-0.4cm}
    \caption{Weighted test AUROC on Erdős-Rényi graphs.}
    \label{fig:erdos}
\end{figure}


We see that in most tasks, the model with random bits achieves the highest test accuracy. However the other models do not perform much worse. In particular, the model using positional embeddings also performs similarly to using no additional node features.\looseness=-1


\subsection{Positional Embeddings vs Ground Truth Positions}

To compare the positional embeddings with ground truth positions, we generate unit disk graphs in the two dimensional euclidean plane. We select $20$ points in the unit square uniformly at random and connect nodes that are within a certain threshold distance to give us our graph. The threshold distance is chosen per task so as to keep the classes balanced. We augment the GNN with ground truth positions and compare against using our own positional embeddings. 

\begin{figure}[ht]
    \centering
    \includegraphics[width=0.9\linewidth]{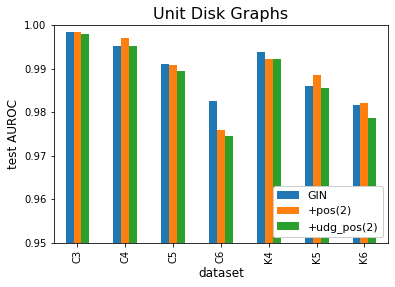}
    \vspace{-0.4cm}
    \caption{Weighted test AUROC on unit disk graphs.}
    \label{fig:udg}
\end{figure}

The results can be seen in Figure \ref{fig:udg}. The positional embeddings perform better than the ground truth positions in almost every task. In several tasks the standard GNN performs the best. This indicates that unit disk graphs can be distinguished by standard GNNs with a high probability.



\subsection{Benchmark Datasets}

\begin{table*}[ht!]
\centering
\resizebox{\linewidth}{!}{%
\begin{tabular}{lllllllll}
\toprule
dataset &         MUTAG &           PTC &      PROTEINS &          NCI1 &        COLLAB &   IMDBBINARY &     IMDBMULTI &  REDDITBINARY \\
\midrule
size            &  188 &  344 & 1113 & 4110 & 5000 & 1000 & 1500 &  2000 \\
classes         &    2 &    2 &    2 &    2 &    3 &    2 &    3 &     2 \\
avg node count  & 17.9 & 25.5 & 39.1 & 29.8 & 74.4 & 19.7 &   13 & 429.6 \\
\midrule
WL subtree \cite{shervashidze2011weisfeilerWL}                  & 90.4 $\pm$ 5.7  &  59.9 $\pm$ 4.3 &  75.0 $\pm$ 3.1 &  86.0 $\pm$ 1.8 &  78.9 $\pm$ 1.9 &  73.8 $\pm$ 3.9 & 50.9 $\pm$ 3.8  &  81.0 $\pm$ 3.1 \\
Invariant Graph Networks \cite{maron2018invariant_equivariant}  & 83.89$\pm$12.95 &  58.53$\pm$6.86 &  76.58$\pm$5.49 &  74.33$\pm$2.71 &  78.36$\pm$2.47 &  72.0$\pm$5.54  & 48.73$\pm$3.41  &  NA         \\
GIN \cite{xu2018powerful_gin}                                   &  89.4$\pm$5.6   &  64.6$\pm$7.0   &  76.2$\pm$2.8   &  82.7$\pm$1.7   &  80.2$\pm$1.9   &  75.1$\pm$5.1   &  52.3$\pm$2.8   &  92.4 $\pm$ 2.5 \\
1-2-3 GNN \cite{morris2019weisfeiler_WL}                        &  86.1$\pm$      &  60.9$\pm$      &  75.5$\pm$      &  76.2$\pm$      &  NA         &  74.2$\pm$      &  49.5$\pm$      &  NA         \\
ppgn \cite{maron2019ppgn}                                       &  90.55$\pm$8.7  &  66.17$\pm$6.54 &  77.2$\pm$4.73  &  83.19$\pm$1.11 &  80.16$\pm$1.11 &  72.6$\pm$4.9   &  50$\pm$3.15    &  NA         \\
DropGNN \cite{papp2021dropgnn}                                  &  90.4 $\pm$7.0  &  66.3 $\pm$8.6  &  76.3 $\pm$6.1  &  NA         &  NA         &  75.7 $\pm$4.2  &  51.4 $\pm$2.8  &  NA         \\
\midrule
GIN$^*$        &  89.33$\pm$5.6 &  63.65$\pm$11.2 &  73.68$\pm$3.7 &  \textbf{82.24$\pm$1.3} &  \textbf{78.00$\pm$2.1} &  74.90$\pm$2.9 &  51.20$\pm$2.8 &  89.95$\pm$1.6 \\
+canon(20)     &  86.14$\pm$7.2 &  61.89$\pm$4.2  &  70.81$\pm$4.8 &  67.42$\pm$2.2 &  74.22$\pm$1.9 &  69.30$\pm$3.9 &  47.13$\pm$2.9 &  79.25$\pm$4.3 \\
+pos(2)        &  83.51$\pm$5.2 &  61.39$\pm$7.7  &  74.21$\pm$3.6 &  74.40$\pm$2.5 &  74.70$\pm$1.9 &  70.60$\pm$3.0 &  49.60$\pm$3.8 &  89.15$\pm$2.0 \\
+pos(2)+rBits(1) &  86.67$\pm$6.3 &  59.04$\pm$7.3  &  73.76$\pm$3.7 &  74.96$\pm$1.6 &  75.20$\pm$1.3 &  71.50$\pm$2.5 &  50.67$\pm$3.3 &  89.00$\pm$1.6 \\
+rBits(2)       &  \textbf{89.80$\pm$7.9} &  \textbf{65.10$\pm$8.3}  &  \textbf{75.38$\pm$3.8} &  81.73$\pm$2.0 &  76.10$\pm$2.2 &  \textbf{75.30$\pm$3.6} &  \textbf{51.87$\pm$2.3} &  \textbf{90.60$\pm$2.1} \\
\bottomrule
\end{tabular}
}
\caption{Test accuracies and standard deviations on the benchmark tasks. $^*$We include scores for our base GIN model for consistency; these are slightly lower than the official values. We highlight the best augmentation and the best overall model.}
\label{tab:real_benchmark}
\end{table*} 

We evaluate our augmented GIN models on real-world graph classification datasets. We compare against the original GIN architecture and various GNN models that aim to go beyond the WL test. We consider four bioinformatics datasets (MUTAG, PTC, PROTEINS, NCI1) and four social network datasets (COLLAB, IMDB-BINARY, IMDB-MULTI, REDDITBINARY) \cite{yanardag2015real_datasets}.

We follow the experimental setup from the original GIN paper \cite{xu2018powerful_gin}. For the social network datasets we use the node degree as the input feature and for the bioinformatics datasets we use the categorical node features supplied with the graphs.
We report the 10-fold cross-validation accuracies, with mean and standard deviation \cite{yanardag2015real_datasets}. 
We use the original 5-layer GIN model (4 layers + input layer) and we apply the most promising augmentations from our synthetic experiments.

Out of the different augmentations, random bits preforms the best on most of the benchmark datasets, outperforming also the base model without additional node features. The augmentations are not state of the art, but they perform competitively on all the tasks. Given how well the WL subtree baseline performs on the datasets, it is possible that classifying graphs in this dataset rarely requires higher expressiveness. As such there is little added value in our augmentations elevating GIN beyond the WL barrier. This is also the case for the other expressive GNNs.


\section{Related Work}

The first works applying neural networks to graphs used recurrent neural networks to learn node representations \cite{gori2005new, scarselli2008graph}. Afterwards many works drew on the success of convolutional neural networks \cite{krizhevsky2012imagenet_convolutions} and generalized this approach to graphs \cite{henaff2015deep, kipf2016gcn, hamilton2017inductive, velivckovic2017graph}. Many of these approaches were later shown to fit into the same general framework of message passing neural networks \cite{gilmer2017quantum_chemistry}. In turn this whole framework was shown to be related to, and indeed upper-bounded by the WL isomorphism test. Although the WL test can distinguish graphs with a very high probability, it has some notable blind spots such as detecting cycles in regular graphs.

In an effort to overcome this boundary, there have been many proposals for increasing the expressive power of GNNs. These include: adding port numbers to edges \cite{sato2019ports_CPNGNN}, adding random features \cite{sato2021random, abboud2020surprising}, randomly dropping nodes or edges \cite{papp2021dropgnn, rong2019dropedge}, adding subgraph counts \cite{bouritsas2020subgraph_counts}, or mimicking higher order WL tests \cite{maron2019ppgn, morris2019weisfeiler_WL}. But many of these either still fail to distinguish simple graphs or are prohibitively expensive. 

Positional encodings also play a major role in NLP. The current state-of-the-art NLP models are based on the permutation invariant attention mechanism \cite{devlin2018bert, lewis2019bart}. But since the relative positions of words is of great importance to meaning, positional embeddings are added to the input tokens to indicate order \cite{shaw2018realtive_pos_encoding_NLP}. 

\section{Conclusion}

We find that positional node embeddings can be very effective for subgraph detection, and can learn to generalize with limited training data. They should be used in applications where subgraph detection is critical, or where the data has inherent spatial meaning.   
Canonical node IDs will be more useful in applications with plenty of data. Random bits can be useful in any scenario and clearly outperform more noisy random features. They can be considered both for increasing expressivity in WL blind spots, and as a form of data augmentation. One could also consider using additional node features only for the ``difficult'' cases. These could be identified by running a WL test before applying a GNN. 




\newpage

\bibliography{example_paper}
\bibliographystyle{icml2022}

\newpage
\appendix
\onecolumn
\section{Proof of Theorem \ref{thm:GNN_canon}}
\begin{theorem}
\label{thm:GNN_canon_b}
GINs of sufficient depth augmented with canonical node IDs can distinguish any two non-isomorphic graphs.
\end{theorem}

The proof follows from Corollary 3.1 in \cite{loukas2019graph_node_ids_proof}. First we state the corollary again for convenience:

\textbf{Corollary 3.1.} $GNN^{n}_{mp}$ can compute any Turing computable function over connected attributed graphs if the following conditions are jointly met: each node is uniquely identified; \textsc(MSG) and \textsc{UP} are Turing-complete for every layer; the depth is at least $d \geq D$ layers; and the width is unbounded.

$D$ denotes the diameter of the graph. $GNN^{n}_{mp}$ refers to a general message passing GNN. \textsc(MSG) and \textsc{UP} are alternative characterizations of the \textsc{aggregate} and \textsc{update} steps. The equivalence of the two models is shown in the paper. And GIN (with sufficient width and depth) is used an example of a network that satisfies the conditions of the theorem.

\begin{proof}
First note that if we can compute any Turing computable function over connected attributed graphs, then in particular we can also solve the graph isomorphism problem. So what remains to be shown is that GIN augmented with canonical node IDs satisfies the condition of the corollary. We already know that GIN with appropriate hyperparameters satisfies the conditions on the GNN. And finally, if we assign a canonical ID to each node, then each node is uniquely identified.
\end{proof}

\section{Proof of Theorem \ref{thm:l_inf}}
\begin{theorem}
\label{thm:l_inf_b}
Every graph $G=(V,E)$ embeds isometrically into $l_{\infty}$.
\end{theorem}

\begin{proof}
Let $V=\{v_1, v_2, \ldots v_n\}$. We define $f_i(v_j) = d(v_i, v_j)$, i.e., as the distance between nodes $v_i$ and $v_j$ in $G$, and we define the mapping of the nodes into $l_{\infty}$ as:
\begin{equation*}
f(v_j) = 
\begin{pmatrix}
f_1(v_j) \\
\ldots \\
f_n(v_j)
\end{pmatrix}
\end{equation*}

Then by the triangle inequality, we have for any $k \in 1,2,\ldots,n$
\begin{equation*}
    \abs{f_k(v_i) - f_k(v_j)} = \abs{d(v_i, v_k) - d(v_k, v_j)} \leq d(v_i, v_j).
\end{equation*}
Therefore,
\begin{equation*}
    \norm{f(v_i) - f(v_j)}_{\infty} \leq d(v_i, v_j).
\end{equation*}
And by construction, considering $k=i$ gives
\begin{equation*}
    \norm{f(v_i) - f(v_j)}_{\infty} = \max_k \abs{f_k(v_i) - f_k(v_j)} \geq \abs{f_i(v_i) - f_i(v_j)} = f_i(v_j) = d(v_i, v_j).
\end{equation*}
Combining these we have 
\begin{equation*}
    \norm{f(v_i) - f(v_j)}_{\infty} = d(v_i, v_j).
\end{equation*}
\end{proof}

\section{Proof of Theorem \ref{thm:GNN_l_inf_universal}}
\begin{theorem}
\label{thm:GNN_l_inf_universal_b}
GIN of sufficient depth augmented with isometric $l_{\infty}$-embeddings can distinguish any two non-isomorphic graphs.\looseness=-1
\end{theorem}

\begin{proof}
The proof follows the same reasoning as the proof of Theorem \ref{thm:GNN_canon_b} above. What remains to be shown is that the $l_{\infty}$-embeddings we defined in the proof of Theorem \ref{thm:l_inf_b} are unique.
However, this is trivial by construction, since for all $v_i, v_j \in V$, we have 
\begin{equation*}
    \norm{f(v_i) - f(v_j)}_{\infty} = d(v_i, v_j) \geq 0.
\end{equation*}
This of course holds for all isometric embedding, since the distance between any two nodes in an unweighted graph is always at least $1$.
\end{proof}

\section{Synthetic Datasets}

Here are the complete results for the synthetic datasets using GIN as the base GNN architecture. We include both test AUROC and test accuracy tables. Scores are always from the final epoch.

\begin{figure}[H]
    \centering
    \includegraphics[width=0.8\linewidth]{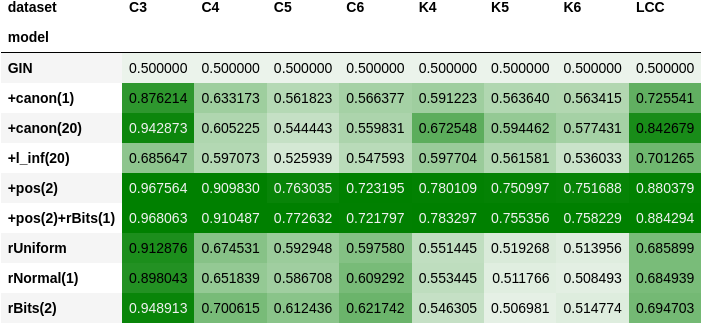}
    \caption{AUROC scores on the synthetic tasks using GIN with 5 layers as the base GNN model.}
    \label{fig:gin_syn_roc_test_full}
\end{figure}

\begin{figure}[H]
    \centering
    \includegraphics[width=0.8\linewidth]{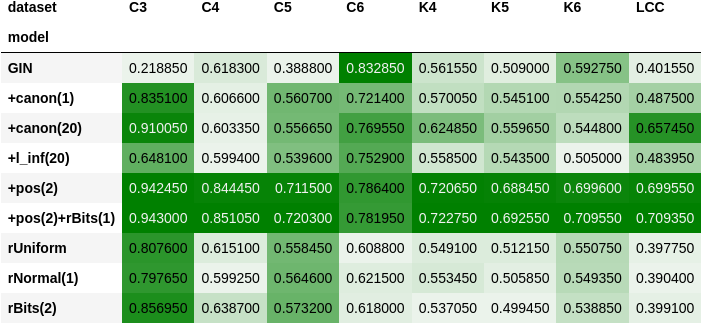}
    \caption{Test accuracies on the synthetic tasks using GIN with 5 layers as the base GNN model.}
    \label{fig:gin_syn_acc_test_full}
\end{figure}

Here are the complete results for the synthetic datasets using GCN as the base GNN architecture. We include both test AUROC and test accuracy tables. Scores are always from the final epoch.

\begin{figure}[H]
    \centering
    \includegraphics[width=0.8\linewidth]{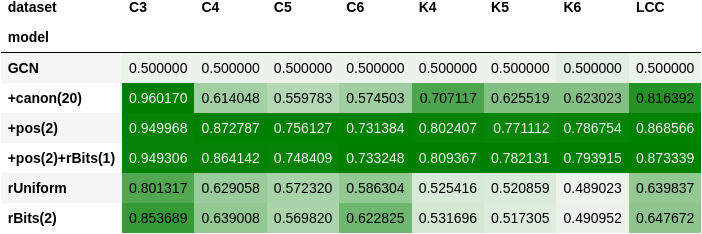}
    \caption{AUROC scores on the synthetic tasks using GCN with 5 layers as the base GNN model.}
    \label{fig:gcn_syn_roc_test_full}
\end{figure}

\begin{figure}[H]
    \centering
    \includegraphics[width=0.8\linewidth]{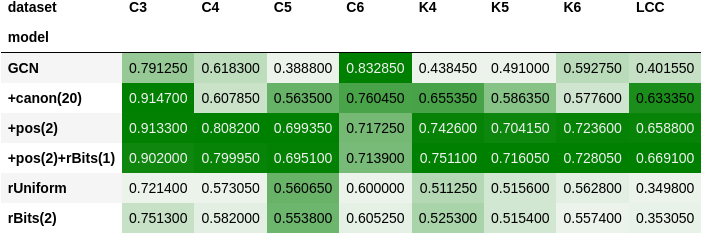}
    \caption{Test accuracies on the synthetic tasks using GCN with 5 layers as the base GNN model.}
    \label{fig:gcn_syn_acc_test_full}
\end{figure}

\section{Further Experiments}

\subsection{Canon has better learning efficiency}

In terms of the test AUROC, canon and pos models level out at around the same time, at epoch $100$, but the pos model at a higher value. The models with the random augmentations level off much later, after epoch $250$.

\begin{figure}[H]
    \centering
    \includegraphics[width=0.6\linewidth]{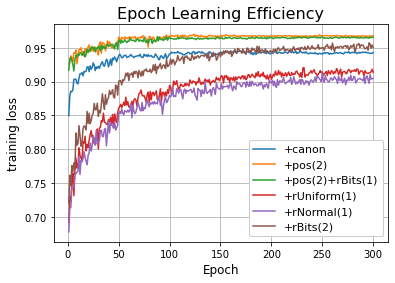}
    \caption{Weighted test AUROC over the first 300 epochs on C3.}
    \label{fig:learning_efficiency_roc_test}
\end{figure}


\end{document}